\documentclass{article}

\usepackage{microtype}
\usepackage{graphicx}
\usepackage{todonotes}
\usepackage{subfigure}
\usepackage{booktabs} 
\usepackage{wrapfig}
\usepackage{enumerate}
\usepackage{amsmath}
\usepackage{hyperref}

\usepackage{amsmath,amsfonts,bm,amsthm}
\usepackage{amssymb}

\newcommand{\norm}[1]{\left\Vert#1\right\Vert}

\newcommand{\abs}[1]{\left\vert#1\right\vert}

\newcommand{\set}[1]{\left\{#1\right\}}

\newcommand{\Real}{\mathbb R}

\newcommand{\eps}{\varepsilon}
\newcommand{\too}{\rightarrow}

\newcommand{\wt}[1]{\widetilde{#1}} 

\makeatletter
\newtheorem*{rep@theorem}{\rep@title}
\newcommand{\newreptheorem}[2]{%
\newenvironment{rep#1}[1]{%
 \def\rep@title{#2 \ref{##1}}%
 \begin{rep@theorem}}%
 {\end{rep@theorem}}}
\makeatother

\newtheorem{theorem}{Theorem}
\newreptheorem{theorem}{Theorem}
\newtheorem{lemma}{Lemma}
\newreptheorem{lemma}{Lemma}

\newtheorem{proposition}{Proposition}

\newtheorem{definition}{Definition}

\makeatletter
\newcommand{\subalign}[1]{%
  \vcenter{%
    \Let@ \restore@math@cr \default@tag
    \baselineskip\fontdimen10 \scriptfont\tw@
    \advance\baselineskip\fontdimen12 \scriptfont\tw@
    \lineskip\thr@@\fontdimen8 \scriptfont\thr@@
    \lineskiplimit\lineskip
    \ialign{\hfil$\m@th\scriptstyle##$&$\m@th\scriptstyle{}##$\crcr
      #1\crcr
    }%
  }
}
\makeatletter

\newcommand{\eg}{{e.g.}}
\newcommand{\ie}{{i.e.}}

\newcommand{\yl}[1]{{\color{magenta}{\bf[Yaron:} #1{\bf]}}}

\def\eqref#1{equation~\ref{#1}}

\def\1{\bm{1}}

\def\eps{{\epsilon}}

\def\mX{{\bm{X}}}

\DeclareMathAlphabet{\mathsfit}{\encodingdefault}{\sfdefault}{m}{sl}
\SetMathAlphabet{\mathsfit}{bold}{\encodingdefault}{\sfdefault}{bx}{n}
\newcommand{\tens}[1]{\bm{\mathsfit{#1}}}

\def\tB{{\tens{B}}}

\def\tL{{\tens{L}}}

\def\tW{{\tens{W}}}
\def\tX{{\tens{X}}}

\def\gL{{\mathcal{L}}}

\def\sN{{\mathbb{N}}}

\newcommand{\R}{\mathbb{R}}
\newcommand{\Z}{\mathbb{Z}}

\usepackage[latin1]{inputenc}
\usepackage[accepted]{icml2019}

\icmltitlerunning{On the Universality of Invariant Networks}

\begin{document}
	
	\twocolumn[
	\icmltitle{On the Universality of Invariant Networks}
	
	
	
	
	\begin{icmlauthorlist}
		\icmlauthor{Haggai Maron}{We}
		\icmlauthor{Ethan Fetaya}{To,Vec}
		\icmlauthor{Nimrod Segol}{We}
		\icmlauthor{Yaron Lipman}{We}
	\end{icmlauthorlist}
	
	\icmlaffiliation{We}{Department of Computer Science and Applied Mathematics, Weizmann Institute of Science, Rehovot, Israel}
	\icmlaffiliation{To}{Department of Computer Science, University of Toronto, Toronto, Canada}
	\icmlaffiliation{Vec}{Vector Institute}

	\icmlcorrespondingauthor{Haggai Maron}{haggai.maron@weizmann.ac.il}

	\icmlkeywords{Machine Learning}
	
	\vskip 0.3in
	
	]
	
	
	
	\printAffiliationsAndNotice{}  
	
	\begin{abstract}
		
		Constraining linear layers in neural networks to respect symmetry transformations from a group $G$ is a common design principle for invariant networks that has found many applications in machine learning.		
		In this paper, we consider a fundamental question that has received little attention to date: Can these networks approximate any (continuous) invariant function? 
		We tackle the rather general case where $G\leq S_n$ (an arbitrary subgroup of the symmetric group) that acts on $\R^n$ by permuting coordinates. This setting includes several recent popular invariant networks. We present two main results: First, $G$-invariant networks are universal if high-order tensors are allowed. Second, there are groups $G$ for which higher-order tensors are unavoidable for obtaining universality. 
		$G$-invariant networks consisting of only first-order tensors are of special interest due to their practical value. We conclude the paper by proving a necessary condition for the universality of $G$-invariant networks that incorporate only first-order tensors.

	\end{abstract}
	
	\section{Introduction}
	The basic paradigm of deep neural networks is repeatedly composing "layers" of linear functions with non-linear, entrywise activation functions to create effective predictive models for learning tasks of interest. 
	
	When trying to learn a function (task) $f$ that is known to be invariant to some group of symmetries $G$ (\ie, $G$-invariant function) it is common to use linear layers that respect this symmetry, namely, invariant and/or equivariant linear layers. Networks with invariant/equivariant linear layers with respect to some group $G$ will be referred here as $G$-\emph{invariant networks}.
	
	A fundamental question in learning theory is that of \emph{approximation} or \emph{universality} \cite{cybenko1989approximation,hornik1991approximation}. In the invariant case: Can a $G$-invariant network approximate an arbitrary continuous $G$-invariant function? 
	
	The goal of this paper is to address this question for \emph{all} finite permutation groups $G\leq S_n$, where $S_n$ is the symmetric group acting on $[n]=\set{1,2,\ldots,n}$. Note that this is a fairly general setting that contains many useful examples (detailed below).
	
	The archetypal example of $G$-invariant networks is Convolutional Neural Networks (CNNs) \cite{lecun1989backpropagation,krizhevsky2012imagenet} that restrict their linear layers to convolutions in order to learn image tasks that are translation invariant or equivariant \footnote{It is common to use convolutional layers without cyclic padding which implies that these networks are not precisely translation invariant. }.

	In recent years researchers are considering other types of data and/or symmetries and consequently new $G$-invariant networks have emerged. Tasks involving point clouds or sets are in general invariant to the order of the input and therefore permutation invariance/equivariance was developed \cite{qi2017pointnet,zaheer2017deep}. Learning tasks involving interaction between different sets, where the input data is tabular, require dealing with different permutations acting independently on each set \cite{hartford2018deep}. Tasks involving graphs and hyper-graphs lead to symmetries defined by tensor products of permutations \cite{kondor2018covariant,maron2018invariant}. A general treatment of invariance/equivariance to finite subgroups of the symmetric group is discussed in \cite{Ravanbakhsh2017};  infinite symmetries are discussed in general in \cite{Kondor2018a} as well as in  \cite{cohen2016group,Cohen2016,Welling2018,Weiler2018}.
	
	Among these examples, universality is known for point-clouds networks and sets networks \cite{qi2017pointnet,zaheer2017deep}, as well as networks invariant to
	finite translation groups (\eg, cyclic convolutional neural networks) \cite{yarotsky2018universal}. However, universality is not known for tabular and multi-set networks \cite{hartford2018deep}, graph and hyper-graph networks \cite{kondor2018covariant,maron2018invariant}; and networks invariant to finite translations with rotations and/or reflections. We cover all these cases in this paper. 
	
	Maybe the most related work to ours is \cite{yarotsky2018universal} that considered actions of compact groups and suggested provably universal architectures that are based on polynomial layers. In contrast, we study the standard and widely used linear layer model. 
	
	The paper is organized as follows:
	First, we prove that an arbitrary continuous function $f:\Real^n\too \Real$ invariant to an arbitrary permutation group $G\leq S_n$ can be approximated using a $G$-invariant network. The proof is constructive and makes use of linear equivariant layers between tensors $\tX\in\Real^{n^k}$ of order $k\leq d$, where $d$ depends on the permutation group $G$. 
	
	Second, we prove a lower bound on the order $d$ of tensors used in a $G$-invariant network so to achieve universality. Specifically, we show that for $G=A_n$ (the alternating group) any $G$-invariant network that uses tensors of order at-most $d=(n-2)/2$  cannot approximate arbitrary $G$-invariant functions. 
	
	We conclude the paper by considering the question: For which groups $G\leq S_n$, $G$-invariant networks using only first order tensors are universal? We prove a necessary condition, and describe families of groups for which universality cannot be attained using only first order tensors.

	

	\section{Preliminaries and main results}
	\label{s:prelim}
	
	The symmetries we consider in this paper are arbitrary subgroups of the symmetric group, \ie,   $G\leq S_n$. The action of $G$ on $x\in \Real^n$ used in this paper is defined as
	\begin{equation}
	g\cdot x = (x_{g^{-1}(1)},\ldots,x_{g^{-1}(n)}), \ g\in G.
	\end{equation}
	The action of $G$ on \emph{tensors} $\tX\in\Real^{n^k \times a}$ (the last index, denoted $j$ represents feature depth) is defined similarly by 
	\begin{equation}\label{e:tensor_equivariance}
	(g\cdot \tX)_{i_1\ldots i_k,j} = \tX_{g^{-1}(i_1)\ldots g^{-1}(i_k),j}, \ g\in G.	
	\end{equation}

	\begin{wraptable}[5]{r}{0.55\columnwidth}
		\vspace{-10pt}\hspace{-15pt}
		\includegraphics[width=0.29\textwidth]{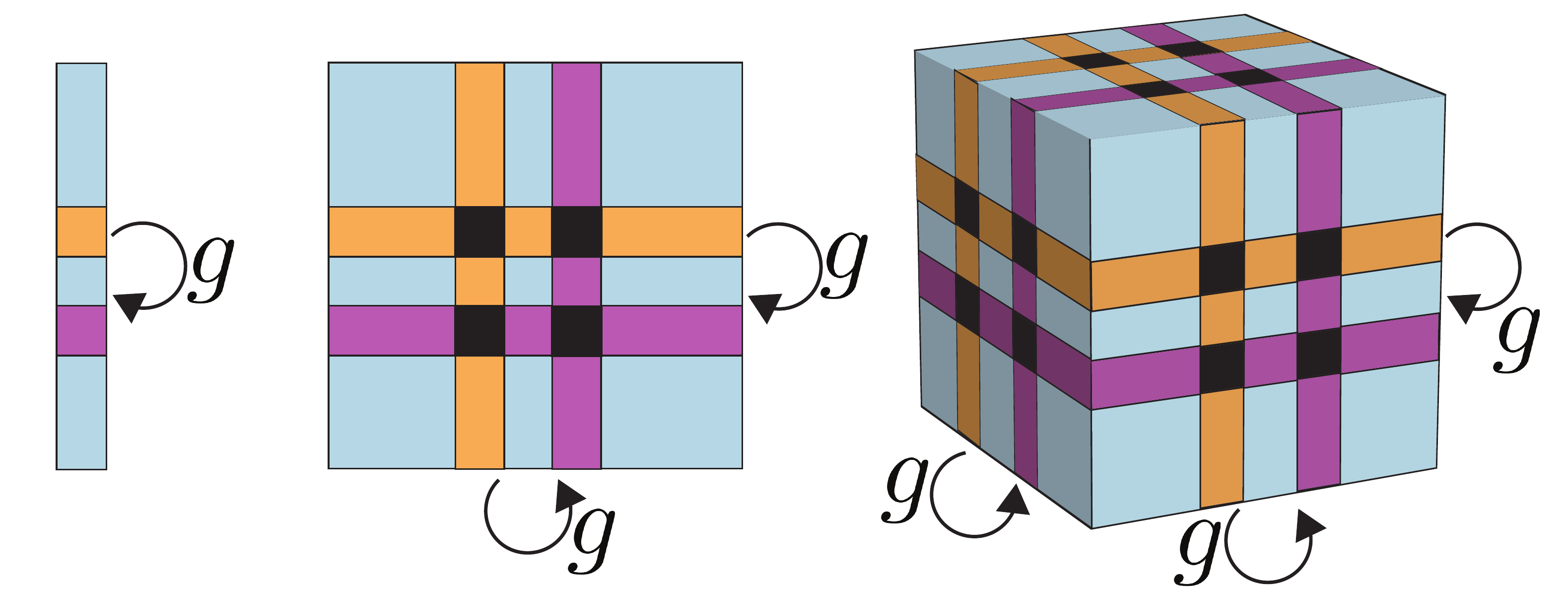}	
	\end{wraptable}
	The inset illustrates this action on tensors of order $k=1,2,3$: the permutation $g$ is a transposition of two numbers and is applied to each dimension of the tensor.
	
	%
	%
	\begin{definition}\label{def:G_invariant_f}
		A $G$-invariant function is a function $f:\Real^n\too\Real$ that satisfies $f(g\cdot x) = f(x)$ for all $x\in\Real^n$ and $g\in G$. 	
	\end{definition}
	
	\begin{definition}\label{def:equi_inv_layer}
		A linear equivariant layer is an affine map $L:\Real^{n^k\times a}\too \Real^{n^{l}\times b}$ satisfying $L(g\cdot \tX) = g\cdot L(\tX)$, for all $g\in G$, and $\tX\in\Real^{n^k\times a}$. An invariant linear layer is an affine map $h:\Real^{n^k\times a}\too \Real^b$ satisfying $h(g\cdot \tX)=h(\tX)$, for all $g\in G$, and $\tX\in\Real^{n^k\times a}$.
	\end{definition}

	A common way to construct $G$-invariant networks is:
	\begin{definition}\label{def:G_invariant_networks}
		A $G$-invariant network is a function $F:\Real^{n\times a} \too\Real$ defined as  $$F= m\circ h\circ L_d\circ\sigma\circ\dots\circ\sigma\circ L_1,$$
		where $L_i$ are linear $G$-equivariant layers, $\sigma$ is an activation function \footnote{We assume any activation function for which the universal approximation theorem for MLP holds, \eg, ReLU and sigmoid.}, $h$ is a $G$-invariant layer, and $m$ is a Multi-Layer Perceptron (MLP).
	\end{definition}
	
	\begin{figure}[t!]
		\centering
		
		\includegraphics{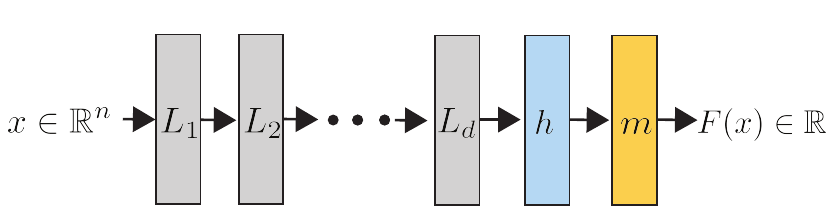} 
		\caption{ Illustration of invariant network architecture. The function is composed of multiple linear $G$-equivariant layers (gray), possibly of high order, and ends with a linear $G$-invariant function (light blue) followed by a Multi Layer Perceptron (yellow).}\label{f:network}
	\end{figure}
	
	Figure \ref{f:network} illustrates the $G$-invariant network model. By construction, $G$-invariant networks are $G$-invariant functions (note that entrywise activation is equivariant as-well). This framework has been used, with appropriate group $G$, in previous works to build predictive $G$-invariant models for learning.

	Our goal is to show the \emph{approximation power} of $G$-invariant networks. Namely, that $G$-invariant networks can approximate arbitrary continuous $G$-invariant functions $f$. Without loss of generality, we consider only functions of the form $f:\Real^n\too \Real$. Indeed, in case of multiple features, $\Real^{n\times a}$, we rearrange the input as $\Real^{n'}$, $n'=na$, and take the appropriate $G'\leq S_{n'}$. We prove:
	\begin{theorem}\label{thm:universality}
		Let $f:\R^n\too\R$ be a continuous $G$-invariant function for some $G\leq S_n$, and $K\subset \R^n$ a compact set. There exists a $G$-invariant network that approximates $f$ to an arbitrary precision. 
	\end{theorem}  
	
	The proof of Theorem \ref{thm:universality} is constructive and builds an $f$-approximating $G$-invariant network with hidden tensors $\tX\in \Real^{n^d}$ of order $d$,  where $d=d(G)$ is a natural number depending on the group $G$. Unfortunately, we show that in the worst case $d$ can be as high as $\frac{n(n-1)}{2}$. Note that $d=2$ could  already be computationally challenging. It is therefore of interest to ask whether there exist more efficient $G$-invariant networks that use lower order tensors without sacrificing approximation power. Surprisingly, the answer is that in general we can not go lower than order $n$ for general permutation groups $G$. Specifically, we prove the following for $G=A_n$, the alternating group:	
	\begin{theorem} \label{thm:lower_bound}
		If an $A_n$-invariant network has the universal approximation property then it consists of tensors of order at least $\frac{n-2}{2}$.
	\end{theorem}

	
	
	Although in general we cannot expect universal approximation of $G$-invariant networks with inner tensor order smaller than $\frac{n-2}{2}$, it is still possible that for \emph{specific} groups of interest we can prove approximation power with more efficient (\ie, lower order inner tensors) $G$-invariant networks. Of specific interest are $G$-invariant networks that use only first order tensors. In section \ref{s:first_order} we prove the following necessary condition for universality of first-order $G$-invariant networks:
	\begin{theorem}\label{thm:necessary}	
		Let $G\leq S_n$. If first order $G$-invariant networks are universal, then $\left| [n]^2/H\right| < \left| [n]^2/G\right| $ for any strict super-group $G < H \leq S_n$. 
	\end{theorem} 
	$|[n]^2/G|$ is the number of equivalence classes of $[n]^2$ defined by the relation: $(i_1,i_2) \sim (j_1,j_2)$ if $j_\ell=g(i_\ell)$, $\ell=1,2$ for some $g\in G$. Intuitively, this condition asks that super-groups of $G$ have \emph{strictly} better separation of the double index space $[n]^2$. 
	
	\section{$G$-invariant networks universality}
	
	The key to showing theorem \ref{thm:universality}, namely that $G$-invariant networks are universal, is showing they can approximate a set of functions that are: (i) $G$-invariant; and (ii) can approximate arbitrary $G$-invariant functions to a desired precision. The $G$-invariant polynomials are an example of such a set: 
	\begin{definition}
		The $G$-invariant polynomials are all the polynomials in $x_1,\dots,x_n$ over $\Real$ that are also $G$-invariant functions. They are denoted  	$\Real[x_1,\ldots,x_n]^G$, where $\Real[x_1,\ldots,x_n]$ is the set of all polynomials over $\Real$.
	\end{definition}
	To see that $G$-invariant polynomials can approximate any arbitrary \emph(continuous) function $f:K\subset \Real^n \too \Real$, where $K$ is a compact set, one can use the Stone-Weiestrass (SW) theorem, as done in \cite{yarotsky2018universal}: First use SW to approximate $f$ over a symmetrized domain $K'=\cup_{g\in G} g\cdot K$ by \emph{some} (not necessarily $G$-invariant) polynomial $p\in\Real[x_1,\ldots,x_n]$. Second, consider $$q(x)=\frac{1}{\abs{G}}\sum_{g\in G} p(g\cdot x).$$ $q$ is a $G$-invariant polynomial and hence $$q\in \Real[x_1,\ldots,x_n]^G,$$ furthermore for $x\in K$:
	\begin{align*}
	&	\abs{q(x)-f(x)} \leq \\ 	&
	\frac{1}{\abs{G}}\sum_{g\in G} \big | p(g\cdot x) - f(g\cdot x) \big | \leq \max_{x\in K'} \abs{p(x)-f(x)}.
	\end{align*}
	Our goal in this section is to prove the following proposition that, together with the comment above, prove theorem \ref{thm:universality}:
	\begin{proposition}\label{thm:main}
		For any $\epsilon>0$, $K\subset\Real^n$ compact set, and $G$-invariant polynomial $p\in\Real[x_1,\ldots,x_n]^G$ there exists a $G$-invariant network $F$ that approximates $p$ to an $\epsilon$-accuracy, namely $\max_{x\in K}\abs{F(x)-p(x)}<\epsilon$. 
	\end{proposition}
	
	The proposition will be proved in several steps:
	\begin{enumerate}[(i)]
		
		\item  We represent $p$ as $p(x)=\sum_{k=0}^d p_k(x)$, where $p_k$ is a $G$-invariant \emph{homogeneous} polynomial of degree $k$, \ie, $p_k\in\Real_k[x_1,\ldots,x_n]^G$.	
		
		\item  We characterize all homogeneous $G$-invariant polynomials of a fixed degree $k$. In particular we find a basis to all such polynomials, $b_{k1},b_{k2},\dots,b_{k n_k}\in \Real_k[x_1,\ldots,x_n]^G$. Using the bases of homogeneous $G$-invariant polynomials of degrees up-to $d$ we write 
		\begin{equation}\label{e:p_decomposition}
		p(x)=\sum_{k=0}^d\sum_{j=1}^{n_k}\alpha_{kj}b_{kj}(x).
		\end{equation}
		
		\item  We approximate each basis element $b_{kj}$ using a $G$-invariant network. 
		
		\item  We construct a $G$-invariant network $F$ approximating $p$ to an $\epsilon$-accuracy using Equation \ref{e:p_decomposition} and (iii).
	\end{enumerate}

	\subsection{Proof of proposition \ref{thm:main}}
	\paragraph{Part (i):}  It is a known fact that a $G$-invariant polynomial can be written as a sum of homogeneous $G$-invariant polynomials \cite{kraft2000classical}:
	\begin{lemma}\label{l:homog_invariant}
		Let $p:\R^n\too\R$ be a $G$-invariant polynomial of degree $d$. Then $p$ can be written as $p(x)=\sum_{k=0}^{d}p_k(x)$ where $p_k$ are homogeneous $G$-invariant polynomials of degree $k$. 
	\end{lemma}

	\paragraph{Part (ii):}
	We need to find bases for the linear spaces of homogeneous $G$-invariant polynomials of degree $k=0,1,\ldots,d$, \ie, $\Real_k[x_1,\ldots,x_n]^G$. 
	Any homogeneous polynomial of degree $k$ can be written as 
	\begin{equation}\label{e:poly}
	p(x)=\sum_{i_1,\ldots,i_k=1}^{n}\tW_{i_1\ldots i_k} \, x_{i_1}\cdots x_{i_k},
	\end{equation}
	where $\tW\in\R^{n^k}$ is its coefficient tensor; since $x_{i_1}\cdots x_{i_k}=x_{i_{\sigma(1)}}\cdots x_{i_{\sigma(k)}}$ for all $\sigma\in S_k$, a unique choice of $\tW$ can be obtained by taking a symmetric $\tW$. That is, $\tW$ that satisfies $\tW_{i_1\cdots i_k}=\tW_{i_{\sigma(1)}\cdots i_{\sigma(k)}}$, for all $\sigma\in S_k$. In short, we ask $\tW \in \text{Sym}^k(\R^n)\subset \Real^{n^k}$. For example, the case $k=2$ amounts to representing a quadratic form using a symmetric matrix, that is $\tW$ satisfies in this case $\tW=\tW^T$. The next proposition shows that if $p$ is $G$-invariant, its coefficient tensor is a fixed point of the action of $G$ on symmetric tensors $\tW\in \R^{n^k}$ :
	\begin{proposition}\label{prop:homogenous_fixed_point} 
		Let $p\in\Real_k[x_1,\ldots,x_n]^G$. Then its coefficient tensor $\tW\in\R^{n^k}$ satisfies the fixed point equation: 
		\begin{equation}\label{e:fixed_point}
		g\cdot \tW = \tW, \ \forall g\in G.	
		\end{equation}
	\end{proposition}
	\begin{proof}
		From the fact that  $p$ is $G$-invariant we get the following set of equations $p(x)=p(g\cdot x)$, for all  $g \in G$. 
		\begin{eqnarray*}
			p(x)  & =& p(g\cdot x)\\ 
			&=& \sum_{i_1,\ldots,i_k=1}^n \tW_{i_1\ldots i_k}\, x_{g^{-1}(i_1)}\cdots x_{g^{-1}(i_k)} \\
			&=& \sum_{i_1,\ldots,i_k=1}^n \tW_{g(i_1)\ldots g(i_k)}\, x_{i_1}\cdots x_{i_k}. 
		\end{eqnarray*}
		By equating monomials' coefficients of $p(x)$ and $p(g\cdot x)$ and the symmetry of $\tW$ we get 
		$$\tW_{i_{1}\ldots i_{k}}=\tW_{g(i_{1})\ldots g(i_{k})}.$$
		This implies that $\tW$ satisfies $g\cdot \tW=\tW$ for all $g\in G$.
	\end{proof}
	
	Equation \ref{e:fixed_point} is a linear homogeneous system of equations and therefore the set of solutions $\tW$ forms a linear space. To define a basis for this linear space we first define the following equivalence relation: $(i_1,\ldots,i_k)\sim (j_1,\ldots,j_k)$ if there exists $g\in G$ and $\sigma\in S_k$ so that $j_\ell=g(i_{\sigma(\ell)})$, $\ell=1,\ldots,k$. Intuitively, $g$ takes care of the $G$-invariance while $\sigma$ factors out the fact that the monomials $x_{i_1}\cdots x_{i_k} = x_{\sigma(i_1)}\cdots x_{\sigma(i_k)}$. For example, let $n=5$, $k=3$, $g=(2 3)(4 5)$, $\sigma=(2 3)$ (we use cycle notation), then we have: $(2,2,4) \sim (3,5,3)$.  
	The equivalence classes are denoted $\tau$ and called the \emph{$k$-classes}.
	%
	%
	We show:
	\begin{proposition}
		The set of polynomials 
		\begin{equation}\label{e:indicator_poly}
		p^\tau(x) = \sum_{(i_1,\dots,i_k)\in\tau} x_{i_1}\cdots x_{i_k},
		\end{equation}
		where $\tau$ is a $k$-class, form a basis to $\Real_k[x_1,\ldots,x_n]^G$. 
	\end{proposition}
	\begin{proof} Denote $\tW^\tau$ the symmetric coefficient tensor of $p^\tau$, Note that 
		\begin{equation}\label{e:thm_basis_W_tau}
		\tW^\tau_{i_1 \ldots i_k} = \begin{cases}
		1 & (i_1,\ldots,i_k)\in \tau\\
		0 & \text{otherwise}
		\end{cases}.
		\end{equation}
		Since $$p^\tau(g\cdot x) = \sum_{(i_1,\ldots,i_k)\in \tau} x_{g^{-1}(i_1)}\cdots x_{g^{-1}(i_k)}=p^\tau(x),$$ $p^\tau\in\Real_k[x_1,\ldots,x_n]^G$. The set of polynomials $p^\tau$, with $\tau$ a $k$-classes, is a linearly independent set since each $p^\tau$ contains a different collection of monomials. 
		By Proposition \ref{prop:homogenous_fixed_point}, the symmetric coefficient tensor $\tW$ of every $q\in \Real_k[x_1,\ldots,x_n]^G$ satisfies the fixed-point equation, \eqref{e:fixed_point}. This in particular means that $\tW$ is constant on its $k$-classes. Hence $\tW$ can be written as linear combination of $\tW^\tau$, see also  \eqref{e:thm_basis_W_tau}. 
	\end{proof}
	
	
	As we later show, the fixed point equation, \eqref{e:fixed_point}, is also used to characterize and compute a basis for the space of \emph{linear} permutation-equivariant and invariant layers \cite{maron2018invariant}.  These equations are equivalently formulated using weight sharing scheme in \cite{Ravanbakhsh2017}. A slight difference in this case, that deals with polynomials, is the additional constraints that formulate the symmetry of $\tW$  which are needed since every polynomial of degree $>1$ has several representing tensors $\tW$.  
	
	\paragraph{Part (iii):} Our next step is approximating each $p^\tau$ with a $G$-invariant network. The next proposition introduces the building blocks of this construction:
	\begin{proposition} \label{e:linear_equi}
		Let $\tau$ be a $k$-class and let $L^{\tau}_\ell:\R^{n}\too \R^{n^{k}}$, $\ell=1,\ldots,k$, be a linear operator defined as follows:\\ For $x\in \Real^n$
		\[ L^\tau_\ell(x)_{i_1\dots i_k}=
		\begin{cases} 
		x_{i_\ell} & (i_1,\dots,i_k)\in\tau \\
		0 & \text{otherwise} 
		\end{cases}.\]
		Then $L^\tau_\ell$ is a linear $G$-equivariant function, that is $$L^\tau_\ell(g\cdot x) = g\cdot L^\tau_\ell(x), \ \forall x\in \Real^n, g\in G.$$
	\end{proposition}
	\begin{proof}
		We have :
		\[g \cdot L^\tau_\ell(x)_{i_1\dots i_k}=
		\begin{cases} 
		x_{g^{-1}(i_\ell)} & (g^{-1}(i_1),\dots,g^{-1}(i_k))\in\tau \\
		0 & \text{otherwise} 
		\end{cases}
		\]
		On the other hand, 
		\[ L^\tau_\ell(g \cdot x)_{i_1\dots i_k}=
		\begin{cases} 
		x_{g^{-1}(i_\ell)} & (i_1,\dots,i_k)\in\tau \\
		0 & \text{otherwise} 
		\end{cases}
		\]
		and both expressions are equal since $(i_1,\dots,i_k)\in\tau$ if and only if  $(g^{-1}(i_1),\dots,g^{-1}(i_k))\in\tau $ by definition of $\tau$.
	\end{proof}
	
	Next, we construct the approximating $G$-invariant network:
	\begin{proposition}\label{prop:approximatin_p_tau}
		For any $\epsilon>0$, $K\subset\Real^n$ compact set, and $\tau$ $k$-class there exists a $G$-invariant network $F^\tau$ that approximates $p^\tau$ from \eqref{e:indicator_poly} to an $\epsilon$-accuracy. 
	\end{proposition}
	\begin{proof}
		Let $c>0$ be sufficiently large so that $K\subset [-c,c]^n \subset \Real^n$. Denote $m^k:\Real^{k}\too \Real$ an MLP that approximates the multiplication function, $f(y_1,\ldots,y_k)=\prod_{i=1}^k y_i$, in $[-c,c]^k$ to $n^{-k}\epsilon$-accuracy, \ie, $\max_{-c\leq y_i \leq c}\abs{f(y)-m^k(y)}<n^{-k}\epsilon$. 	
		
		Consider the following $G$-invariant network: First, given an input $x\in \Real^n$ map it to $\Real^{n^k\times k}$ (\ie, $k$ is the number of channnels) by 
		\begin{equation}\label{e:L_tau}
		L^\tau(x)_{i_1\ldots i_k, \ell}=L^\tau_\ell(x)_{i_1\ldots i_k}.
		\end{equation}
		$L^\tau:\Real^n \too \Real^{n^k\times k}$ is a linear equivariant layer (see \eqref{e:tensor_equivariance}). Second, apply $m^k$ to the feature dimension in $\Real^{n^k\times k}$. That is, given $y\in\Real^{n^k\times k}$ define $$M^k(y)_{i_1,\ldots,i_k}=m^k(y_{i_1\ldots, i_k,1},\ldots,y_{i_1\ldots, i_k,k}).$$ Note that $M^k:\Real^{n^k\times k}\too \Real^{n^k}$ can be interpreted as a composition of equivariant linear layers \footnote{In fact, any application of an MLP to the feature dimension is $G$-equivariant for any $G\leq S_n$ since it can be realized by  scaling of the identity operator, possibly with a constant and non-linear point-wise activations (see \eg \cite{qi2017pointnet,zaheer2017deep}).}.
		
		Lastly, denote $s:\Real^{n^k}\too \Real$ the summation layer: for $z\in\Real^{n^k}$, $s(z)=\sum_{i_1\ldots i_k = 1}^n z_{i_1\ldots i_k}$. Note that $M^k,s$ are equivariant, invariant (respectively) for all $G\leq S_n$. This construction can be visualized using the following diagram:
		$$\Real^n\xrightarrow{L^{\tau}} \Real^{n^k\times k}\xrightarrow{M^k}\Real^{n^k}\xrightarrow{s} \Real$$

		This $G$-invariant network $F^\tau=s\circ M^k \circ L^\tau$ 
		approximates $p^\tau$ to an $\epsilon$-accuracy over the compact set $K\subset \Real^n$. Indeed, let $x\in K$, then
		\begin{align*}
		& \abs{F^\tau(x)-p^\tau(x)}\\
		& \leq \sum_{i_1\ldots i_k=1}^n\abs{M^k(L^\tau(x))_{i_1\ldots i_k} - \tW^\tau_{i_1\ldots i_k}x_{i_1}\cdots x_{i_k}} \\ 
		& \leq \sum_{i_1\ldots i_k=1}^n \begin{cases} \abs{m^k(x_{i_1},\ldots,x_{i_k})-x_{i_1}\cdots x_{i_k}} & {\scriptstyle (i_1,\ldots,i_k)\in \tau } \\ 0 & {\text{otherwise}} \end{cases} \\
		& \leq \epsilon, \end{align*} 
		where in the last inequality we used the $n^{-k}\epsilon$-accuracy of $m^k$ to the product operator in $[-c,c]^k\subset \Real^k$. 	
		
	\end{proof}

	\paragraph{Part (iv):}
	In the final stage, we would like to approximate an arbitrary $p\in \Real[x_1,\ldots,x_n]^G$ with a $G$-invariant network to $\epsilon$-accuracy over a compact set $K\subset \Real^n$.  
	\begin{proof} \textit{(proposition \ref{thm:main})}
		Let us denote by $b_{k1},\dots,b_{kn_k}$ the polynomials $p^\tau$, with $\tau$ the $k$-classes. Let $F^{kj}$ denote the $G$-invariant network approximating $b_{kj}$, $k=0,1,\ldots d$, $j\in [n_k]$, to an $\epsilon$-accuracy over the set $K$, the existence of which is guaranteed by proposition \ref{prop:approximatin_p_tau}. We now utilize the decomposition of $p$ shown in  \eqref{e:p_decomposition} and get 
		\begin{align*}
		& \abs{p(x)-\sum_{k=0}^d\sum_{j=1}^{n_k}\alpha_{kj}F^{kj}(x)}  \\
		&\leq \sum_{k=0}^d\sum_{j=1}^{n_k}\abs{\alpha_{kj}} \abs{b_{kj}(x) - F^{kj}(x)} \\ &\leq \epsilon \norm{\alpha}_1,
		\end{align*}
		where $\norm{\alpha}_1=\sum_{k,j}|\alpha_{kj}|$ depends only upon $p$, where $\epsilon$ is arbitrary. 
		To finish the proof we need to show that $F=\sum_{k=0}^d\sum_{j=1}^{n_k}\alpha_{kj}F^{kj}$ can indeed be realized as a \emph{single, unified} $G$-invariant network. This is a simple yet technical construction and we defer the proof of this fact to the supplementary material:
		\begin{lemma}\label{lem:unifying_G}
			There exists a $G$-invariant network in the sense of definition \ref{def:G_invariant_networks} that realizes the sum of $G$-invariant networks $F=\sum_{k=0}^d\sum_{j=1}^{n_k}\alpha_{kj}F^{kj}$.
		\end{lemma}

	\end{proof}

	\subsection{Bounded order construction}
	We have constructed a $G$-invariant network $F$ that approximates an arbitrary $G$-invariant polynomial $p\in\Real[x_1,\ldots,x_n]^G$ of degree $d$. The network $F$ uses $d$-dimensional tensors, where $d$ matches the degree of $p$. In this subsection we construct a $G$-invariant network $F$ that approximates $p$ with maximal tensor order that depends only on the group $G\leq S_n$. Therefore, the tensor order is independent of the degree of the polynomial $p$. We use the following theorem by Noether \cite{kraft2000classical}:
	\begin{theorem} \textbf{(Noether)}\label{thm:noether}
		Let $G$ be a finite group acting linearly on $\R^n$.	There exist finitely many $G$-invariant polynomials $f_1, . . . , f_m\in\Real[x_1,\ldots,x_n]^G$ such that any invariant polynomial  $p\in \Real[x_1,\ldots,x_n]^G$ can be expressed as 	$$ p(x) = h(f_1(x), . . . , f_{m}(x)),$$
		where $h\in\Real[x_1,\ldots,x_m]$ is a polynomial and $\deg(f_i)\leq |G|$, $i=1,\ldots,m$.
	\end{theorem}   
	The idea of using a set of generating invariant polynomials in the context of universality was introduced in \cite{yarotsky2018universal}.
	
	For the case of interest in this paper, namely $G\leq S_n$,  there exists a generating set of $G$-invariant polynomials of degree bounded by $\frac{n(n-1)}{2}$, for $n\geq 3$, see \cite{gobel1995computing}. We can now repeat the construction above, building a $G$-invariant network $F_i$ approximating $f_i$ to a $\epsilon_1$-accuracy, $i=1,\ldots,m$. The maximal order of these networks is bounded by $d \leq \frac{n(n-1)}{2}$. These networks can be combined, as above, to a single $G$-invariant network $F:\Real^n\too\Real^m$ with the final output approximating $f(x)=(f_1(x),\ldots,f_m(x))$ to a $\epsilon_1$-accuracy. Now we compose the output of $F$ with an MLP $H:\Real^m\too \Real$ approximating the polynomial $h$ over the compact set $f(K)+B_\epsilon\subset \Real^m$  to an $\epsilon$-accuracy, where $B_\epsilon$ is a closed ball centered at the origin of radius $\epsilon$ and the sum is the Minkowski sum. Since $H$ is continuous and $f(K)+B_\epsilon$ is compact, there exists $\delta>0$ so that $\abs{H(y)-H(y')}\leq \epsilon$ if $\norm{y-y'}_2\leq \delta$. We use $\epsilon_1=\min \set{\delta,\eps}$ for the construction of $F$ above. We have:
	\begin{align*}
	\abs{H(F(x))-h(f(x))}  & \leq \abs{H(F(x))-H(f(x))} \\ &+ \abs{H(f(x))-h(f(x))} \leq 2\epsilon,
	\end{align*}
	for all $x\in K$. We have constructed $H\circ F$ that is a $G$-invariant network with maximal tensor order bounded by $\frac{n(n-1)}{2}$ approximating $p$ to an arbitrary precision.

	\subsection{Examples} 
	
	\paragraph{Universality of (hyper-) graph networks.}
	Graph, or hyper-graph data can be described using tensors $\tX\in\Real^{n^k\times a}$, where $n$ is the number of vertices of the graph and $x_{i_1,i_2,\ldots,i_k,:}\in\Real^a$ is a feature vector attached to a (generalized-)edge defined by the ordered set of vertices $(i_1,i_2,\ldots,i_k)$. For example, an adjacency matrix of an $n$-vertex graph is described by $\mX\in\Real^{n^2}$. The graph symmetries are reordering the vertices by a permutation, namely $g\cdot \tX$, where $g\in S_n$. Typically, any function we would like to learn on graphs would be invariant to this action.  Recently, \cite{maron2018invariant} characterized the spaces of equivariant and invariant linear layers with this symmetry, provided a formula for their basis and employed the corresponding $G$-invariant networks for learning graph-related tasks.  A corollary of Theorem \ref{thm:main} is that this construction yields a universal approximator of continuous functions defined on graphs. This is in contrast to the popular \emph{message passing neural network} model \cite{Gilmer2017} that was recently shown to be non-universal \cite{xu2018how}.
	
	\paragraph{Universality of rotation invariant convolutional networks.} For learning tasks involving $m\times m$ images one might require invariance to periodic translations and $90$ degree rotations. Note that periodic translations and $90$ degree rotations can be seen as permutations in $S_n$, $n=m^2$, acting on the pixels of the image. Constructing a suitable $G$-invariant network would lead, according to Theorem \ref{thm:main}, to a universal approximator. 
	
	\section{A lower bound on equivariant layer order}\label{s:lower_bound}
	In the previous section we showed how an arbitrary $G$-invariant polynomial can be approximated with a $G$-invariant network with tensor order $d=d(G)\leq \frac{n(n-1)}{2}$. This upper-bound would be prohibitive in practice. In this section we prove a \emph{lower bound}: We show that there exists a group for which the tensor order cannot be less than $\frac{n-2}{2}$ if we wish to maintain the universal approximation property.

	\begin{figure}[t!]
		\centering
		
		\includegraphics[width=0.4\textwidth]{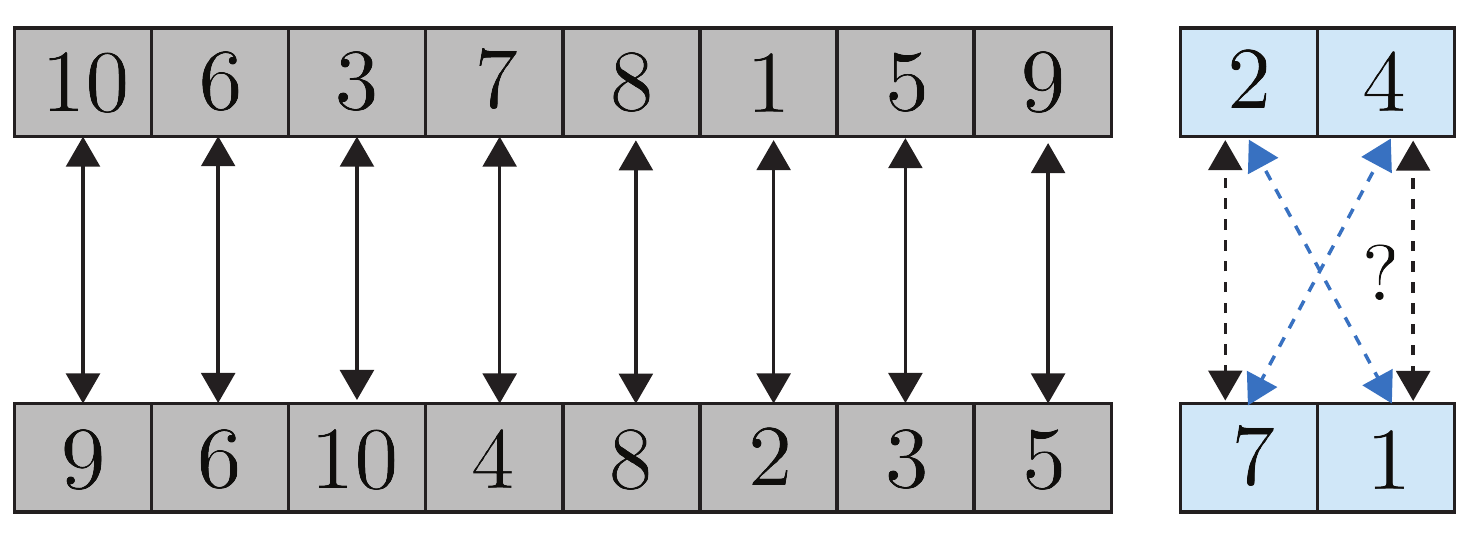} 
		\caption{ Illustration of the $(n-2)$-transitivity of $A_n$, the main property we use in this section. Any subset of distinct $n-2$ elements can be mapped to any other subset of distinct $n-2$ elements (gray). If needed, a transposition can be applied to the remaining $2$ elements (blue) to assure an even permutation. }\label{f:transitive}
	\end{figure}

	We consider the alternating group, $G=A_n\leq S_n$. Remember that $g\in A_n$ if $g$ has an even number of transpositions. 
	\begin{definition}
		A group $G\leq S_n$ is $k$-transitive if for every two sequences $(i_1,i_2,\ldots,i_k)$, $(j_1,j_2,\ldots,j_k)$ of distinct elements in $[n]$ there exists $g\in G$ so that $j_\ell=g(i_\ell)$, for $\ell=1,\ldots,k$.	
	\end{definition}
	The alternating group is $(n-2)$-transitive (see figure \ref{f:transitive} and \cite{dixon1996permutation}). Our goal is to prove:
	\begin{reptheorem}{thm:lower_bound}
		If an $A_n$-invariant network has the universal approximation property, then it consists of tensors of order at least $\frac{n-2}{2}$.
	\end{reptheorem}

	For the proof we first need a characterization of the linear equivariant layers $L:\Real^{n^k\times a} \too \Real^{n^l\times b}$, where $l=0$ represents the invariant case. By definition $L(g\cdot \tX) = g\cdot L(\tX)$ for all $\tX\in\Real^{n^k\times a}$.  In particular this means that $$g^{-1}\cdot L ( g\cdot \tX) = L(\tX)$$
	Recall that $L$ is an \emph{affine map} (see definition \ref{def:equi_inv_layer}) and therefore can be represented as a sum of a purely linear part and a constant part. 
	Representing the linear part of $L$ as a tensor $\tL\in\Real^{n^{k+l}\times a \times b}$  these equations become the fixed-point equation for linear equivariant layers  (see supplementary material for derivation): 
	\begin{equation}\label{e:equi_fixed_point_linear}
	g\cdot \tL = \tL, \ g\in G.	
	\end{equation}
	The constant part of $L$ can be encoded using a tensor $\tB\in\Real^{n^l\times b}$ that satisfies \eqref{e:equi_fixed_point_linear} as-well.  
	%
	Note the that this fixed point equation is similar to the fixed point equation of homogeneous $G$-invariant polynomials, \eqref{e:fixed_point}. We denote by $\mathcal{L}^G$  the collection of $L:\Real^{n^k\times a}\too \Real^{n^l\times b}$ linear $G$-equivariant ($l>0$) or $G$-invariant ($l=0$) layers.  
	
	\begin{proposition}\label{prop:Sn_An_equi}
		If $k+l\leq n-2$, then $\mathcal{L}^{A_n}=\mathcal{L}^{S_n}$. 
	\end{proposition}
	\begin{proof} 
		In view of the fixed point equation for equivariant/invariant layers (\ref{e:equi_fixed_point_linear}) we need to show the solution set to this equation is identical for $G=A_n$ and $G=S_n$, as long as $k+l\leq n-2$. The solution set of the fixed point equation consists of tensors $\tL$ that are constant on each equivalence class defined by the equivalence relation: $(i_1,\ldots,i_{k+l})\sim (j_1,\ldots,j_{k+l})$ if $j_\ell = g(i_\ell)$ for $\ell=1,\ldots,k+l$. 
		
		Both $A_n$ and $S_n$ are $(n-2)$-transitive\footnote{$S_n$ is in-fact $n$-transitive and is therefore also $k$-transitive for all $k\leq n$.}. Therefore, the equivalence relations defined above for $A_n$ and $S_n$ reduce to the \emph{same} equivalence relation $(i_1,\ldots,i_{k+l})\sim (j_1,\ldots,j_{k+l})$  if $i_\alpha=i_\beta$ if and only if $j_\alpha=j_\beta$, for all $\alpha,\beta \in [k+l]$ (see \cite{maron2018invariant} where these classes are called \emph{equality patterns}). Since this equivalence relation is the same for $A_n,S_n$, we get that the solution set of the fixed point equation (\ref{e:equi_fixed_point_linear}) is the same for both groups. Since the constant part tensor $\tB$ is of smaller order than $k+l\leq n-2$, the same argumentation applies to the constant part, as-well. 
	\end{proof}
	Proposition \ref{prop:Sn_An_equi} implies that any $A_n$-invariant network with tensor order $\leq(n-2)/2$ will be in fact $S_n$-invariant. Therefore, one approach to show that such networks have limited approximation power is to come up with an $A_n$-invariant continuous function that is \emph{not} $S_n$-invariant, as follows: 
	\begin{proof}\textit{(Theorem \ref{thm:lower_bound})}
		Consider the Vandermonde polynomial $V(x)=\prod_{1\leq i<j \leq n}(x_i-x_j)$. It is not hard to check that $V$ is $A_n$-invariant but not $S_n$-invariant (consider, \eg, $g=(12)\in S_n$). Pick $x\in\R^n$ with distinct coordinates. Then it holds that $V(x)\ne 0$. Let $\epsilon>0$ and $K\subset \Real^n$ a compact set containing both $x,g\cdot x$ for $g=(12)$. Assume by way of contradiction that there exists an $A_n$-invariant network $F$, which is $S_n$-invariant due to the above,  such that $|V(x)-F(x)|\leq \epsilon$ as well as:
		\begin{align*}
		|V(g\cdot x)-F(g\cdot x)|&=|(-1)V(x)-F(x)|\\ &=|V(x)+F(x)|\\&\leq\epsilon 
		\end{align*}
		These last equations imply that $|V(x)|\leq\epsilon$ and since $\epsilon$ is arbitrary we get $V(x)=0$, a contradiction.
	\end{proof}

	\section{Universality of first order networks}\label{s:first_order} 
	We have seen that $G$-invariant networks with tensor order $\frac{n(n-1)}{2}$ are universal. On the other hand for general permutation groups $G$ the tensor order is at least $(n-2)/2$ if universality is required. A particularly important question for applications, where higher order tensors are computationally prohibitive, is which permutation groups $G$ give rise to \emph{first order $G$-invariant networks} that are universal.  
	\begin{definition}\label{def:first_order_inv_networks}
		A first order $G$-invariant network is a $G$-invariant network where the maximal tensor order is $1$.	
	\end{definition}
	In this section we discuss this (mostly) open question.
	First, we note that there are a few cases for which first order $G$-invariant networks are known to be universal: for instance, when $G= \left\lbrace e\right\rbrace$ (\ie, the trivial group), $G$-invariant networks are composed of fully connected layers, a case which is covered by the original universal approximation theorems \cite{cybenko1989approximation,hornik1991approximation}.  First order universality is also known when $G$ is (possibly high dimensional) grid (\eg, $G=\Z_{n_1}\times \dots \times \Z_{n_k}$) \cite{yarotsky2018universal}, a case that includes periodic convolutional neural networks. Universality of first order networks is also known when $G=S_n$ \cite{zaheer2017deep,qi2017pointnet,yarotsky2018universal} in the context of invariant networks that operate on sets or point clouds.  
	
	Our goal in this section is to derive a necessary condition on $G$ for the universality of first order $G$-invariant networks. To this end, we first find a function, playing the role of the Vandermonde polynomial in the previous section, that is $G$-invariant but not $H$-invariant, where $G < H \leq S_n$. 
	\begin{lemma}\label{l:separating_function}
		Let $G<H\leq S_n$. Then there exists a continuous function $f:\R^n\too \R$ which is $G$-invariant but not $H$-invariant.
	\end{lemma}
	\begin{proof}
		Pick a point $x_0\in\R^n$ with distinct coordinates. Since the stabilizer $(S_n)_{x_0}$ is trivial (\ie, no permutation fixes $x_0$ excluding the identity), the size of the orbits of $x_0$ equals the size of the acting group. Namely, $|G\cdot x_0|=|G|$ and $|H \cdot x_0|=|H|$. Furthermore, since $|G|<|H|$ and $G \cdot x_0\subset H \cdot  x_0$, we get that the $H$ orbit strictly includes the $G$ orbit. That is, $G\cdot x\subsetneq H\cdot x$. Since $H\cdot x_0$ is a finite set of points, there exists a continuous function $\hat{f}$ such that  $\hat{f}|_{G\cdot x_0}=1$, and $\hat{f}|_{H\cdot x_0\setminus G\cdot x_0}=0$. Define $f(x)=\frac{1}{|G|}\sum_{g\in G} \hat{f} (g\cdot x)$. Now, $f$ is $G$-invariant by construction but $f(x_0)=1$ and $f(h\cdot x_0)=0$ for $h\cdot x_0 \in H\cdot x_0 \setminus G\cdot x_0$. Therefore, $f$ is not $H$-invariant.
	\end{proof}

	In case of first order $G$-invariant networks the equivariant/invariant layers have the form $L:\Real^{n\times a} \too \Real^{n\times b}$ and satisfy the fixed point equations (\ref{e:equi_fixed_point_linear}). The solution set of the purely linear equivariant layers consists of tensors $\tL\in\Real^{n^2\times a \times b}$ that are constant on equivalence classes of indices defined by the equivalence relation $(i_1,i_2)\sim (j_1,j_2)$ if there exists $g\in G$ so that $j_\ell=g(i_\ell)$, $\ell=1,2$. We denote the number of equivalence classes by $|[n]^2/G|$. The solution set of constant equivariant operators are tensors $\tB\in\Real^{n\times b}$ that are constant on equivalence classes defined by the equivalence relation $i\sim j$ if there exists $g\in G$ so that $j=g(i)$. We denote the number of these classes by $|[n]/G|$. We prove:
	
	\begin{reptheorem}{thm:necessary}	
		Let $G\leq S_n$. If first order $G$-invariant networks are universal, then $\left| [n]^2/H\right| < \left| [n]^2/G\right| $ for any strict super-group $G < H \leq S_n$. 
	\end{reptheorem}  
	\begin{proof}
		Assume by contradiction that there exists a strict super-group $G<H\leq S_n$ so that $\left| [n]^2/G\right| = \left| [n]^2/H\right|$. This in particular means that $|[n]/G|=|[n]/H|$. Therefore $\gL^G=\gL^H$. That is, the spaces of equivariant and invariant linear layers coincide for $G$ and $H$. This implies, as before, that every first order $G$-invariant network is also $H$-invariant.
		
		We proceed similarly to the proof of theorem \ref{thm:lower_bound}: By lemma \ref{l:separating_function}, there exists a continuous function $f:\Real^n\too\Real$ that is $G$-invariant but not $H$-invariant. Let $x_0$ be a point with distinct coordinates where $f(x_0)=1$ (it exists by construction, see proof of theorem \ref{thm:lower_bound}). Furthermore, by construction $f(h\cdot x_0) = 0$ if $h\cdot x_0 \in H\cdot x_0 \setminus G \cdot x_0$.

		Let $\epsilon>0$ and $K\subset \Real^n$ a compact set containing both $x_0, h\cdot x_0$. Assume by way of contradiction that there exists a first order $G$-invariant network $F$ (which is also $H$-invariant in view of the above) such that $|f(x_0)-F(x_0)|\leq \epsilon$ as well as:
		\begin{align*}
		|f(h\cdot x_0)-F(h\cdot x_0)|&=|f(h\cdot x_0)-F(x_0)| \leq\epsilon.
		\end{align*}
		These last equations imply that $1 = |f(x_0)-f(h\cdot x_0)|\leq |f(x_0)-F(x_0)| + |F(x_0)-f(h\cdot x_0)|\leq 2\epsilon$ and since $\epsilon$ is arbitrary we get a contradiction.
	\end{proof}

	Using theorem \ref{thm:necessary} we can show that there exist a few infinite families of permutation groups (excluding the alternating group $A_n$) for which first order invariant networks are not universal. For example, any strict subgroup $G < S_n$ that is $2$-transitive is such a group since in this case $\left| [n]^2/G\right| = \left| [n]^2/S_n\right| $ and consequently $G$-invariant/equivariant layers are also $S_n$-invariant/equivariant. Examples of 2-transitive permutation groups include projective linear groups over finite fields $PSL_d(F_q)$ (for $q=p^n$ where $p,n\in\sN$, $p$ is prime) that act on the finite projective space, and can be seen as a subgroup of $S_n$ for $n=(q^d-1)/(q-1)$ (the number of elements in this space ). Similarly affine subgroups over finite fields $A\Gamma L_d(F_q)$  that act on $F_q^d$ can be shown to be $2$-transitive as a subgroup of $S_n$ for $n=q^d$. See \cite{dixon1996permutation} for a full classification of $2$-transitive subgroups of $S_n$.

	\paragraph{Relation to \cite{Ravanbakhsh2017}.}
	Groups for which the condition in theorem \ref{thm:necessary} holds are called 2-closed and were first introduced by \cite{wielandt1969permutation} (see \cite{babai1995automorphism} for further study). Theorem \ref{thm:necessary} reveals an interesting connection between our work and the work of \cite{Ravanbakhsh2017} that studies parameter sharing schemes. One of the basic notions defined in their paper is the notion of \emph{uniquely $G$-equivariant functions}, which describes functions that are $G$-equivariant but not equivariant to any super-group of $G$.  For example, a consequence of proposition \ref{prop:Sn_An_equi} is that $A_n\leq S_n$ (with the representation used in this paper) has no uniquely equivariant linear functions  between tensors of total order $\leq n-2$. It was shown in \cite{Ravanbakhsh2017} that 2-closed groups are exactly the groups for which one can find a uniquely equivariant function. In this section we proved that the existence of a uniquely $G$-equivariant linear function is a necessary condition for first order universality. As stated in \cite{Ravanbakhsh2017} some examples for 2-closed groups are fixed-point free groups (\eg, the cyclic group $C_n$) and $S_n$ itself.

	\section{Conclusion}
	In this paper we have considered the universal approximation property of a popular invariant neural network model. We have shown that these networks are universal with a construction that uses tensors of order $\leq \frac{n(n-1)}{2}$, which makes this architecture impractical. On the other hand, there exists a permutation group for which we have proved a lower bound of $\frac{n-2}{2}$ on the tensor order required to achieve universality. We then addressed the more practical question of which groups $G$ allow first order $G$-invariant networks to be universal. We have proved that 2-closedness of $G$ is a necessary condition, and gave examples of infinite permutation group families that do not satisfy this condition. 
	
	Our work is a first step in advancing the understanding of approximation power of a large class of invariant neural networks that becomes increasingly popular in applications. Several questions remain open: First, a classification of 2-closed groups will give us a complete answer to which networks are first-order universal. As far as we know this is an open question in group theory. Still, mapping the 2-closed landscape for specific groups $G$ that are interesting for machine learning applications is a worthy challenge. Second, In case one wishes to construct a $G$-invariant network for a group $G$ that is not 2-closed, developing fast and efficient implementations of higher order layers seems like a potentially useful direction. Lastly, another interesting venue for future work might be to come up with new, possibly non-linear, models for invariant networks.     
	
	\subsection*{Acknowledgments}
	This research was supported in part by the European Research Council (ERC Consolidator Grant, "LiftMatch" 771136) and the Israel Science Foundation (Grant No. 1830/17).

	\bibliography{example_paper}
	\bibliographystyle{icml2019}
\clearpage

 \appendix
 \section{Proofs}
 
 \begin{replemma}{lem:unifying_G}
 	There exists a $G$-invariant network in the sense of definition \ref{def:G_invariant_networks} that realizes the sum of $G$-invariant networks $F=\sum_{k=0}^d\sum_{j=1}^{n_k}\alpha_{kj}F^{kj}$.
 \end{replemma}
 \begin{proof}
 	We need to show that $F=\sum_{k=0}^d\sum_{j=1}^{n_k}\alpha_{kj}F^{kj}$ can indeed be realized as a \emph{single, unified} $G$-invariant network. As we already saw, each network $F^{kj}$ has the structure $$\Real^n\xrightarrow{L^{\tau}} \Real^{n^k\times k}\xrightarrow{M^k}\Real^{n^k}\xrightarrow{s} \Real,$$ with a suitable $k$-class $\tau$. To create the unified $G$-invariant network we first lift each $F^{kj}$ to the maximal dimension $d$. That is, $\wt{F}^{kj}$ with the structure  $$\Real^n\xrightarrow{\wt{L}^{kj}} \Real^{n^d\times k}\xrightarrow{\wt{M}^k}\Real^{n^d}\xrightarrow{s} \Real.$$ This is done by composing each equivariant layer $L:\Real^{n^k\times a}\too \Real^{n^l\times b}$ with two linear equivariant operators $U^b:\Real^{n^k\times b}\too \Real^{n^d \times b}$ and $D^a:\Real^{n^d \times a}\too \Real^{n^k\times a}$, 
 	\begin{equation}\label{e:ULD}
 	U^bLD^a: \Real^{n^d\times a} \too \Real^{n^d\times b},	
 	\end{equation}
 	where
 	$$U^b(x)_{i_1\ldots i_d,j} = x_{i_1\ldots i_k,j}$$
 	and
 	$$D^a(y)_{i_1\ldots i_k,j} = n^{k-d}\sum_{i_{k+1}\ldots i_d = 1}^n y_{i_1 \ldots i_k i_{k+1} \ldots i_d,j}\ . $$
 	Since $U^b,D^a$ are equivariant, $U^bLD^a$ in \eqref{e:ULD} is equivariant. Furthermore $D^a \circ \sigma \circ U^a = \sigma$, where $\sigma$ is the pointwise activation function.  Lastly, given two $G$-invariant networks with the same tensor order $d$ they can be combined to a single $G$-invariant network by concatenating their features. That is, if $L_1:\Real^{n^d\times a}\too \Real^{n^d\times b}$, and $L_2:\Real^{n^d\times a'}\too \Real^{n^d\times b'}$, then their concatenation would yield $L_{1,2}:\Real^{n^d\times (a +a')}\too \Real^{n^d\times (b+b')}$. Applying this concatenation to all $\wt{F}^{kj}$ we get our unified $G$-invariant network. 
 \end{proof}
 
 \paragraph{Fixed-point equation for equivariant layers.}
 We have an affine operator $L:\Real^{n^k\times a} \too \Real^{n^l\times b}$ satisfying 
 \begin{equation}\label{e:supp_g_L_g}
 g^{-1}\cdot L ( g\cdot \tX) = L(\tX),
 \end{equation}
 for all $g\in G$, $\tX\in\Real^{n^k \times a}$. The purely linear part of $L$ can be written using a tensor $\tL\in\Real^{n^{k+l}\times a \times b}$: Write $$L(\tX)_{j_1 \ldots j_l ,j} = \sum_{i_1\ldots i_k, i} \tL_{j_1\ldots j_l, i_1 \ldots i_k, i, j} \tX_{i_1\ldots i_k, i}.$$
 Writing \eqref{e:supp_g_L_g} using this notation gives:
 \begin{align*}
 &\sum_{i_1\ldots i_k, i} \tL_{g(j_1)\ldots g(j_l), i_1 \ldots i_k, i, j} \tX_{g^{-1}(i_1)\ldots g^{-1}(i_k), i} \\ &= \sum_{i_1\ldots i_k, i} \tL_{g(j_1)\ldots g(j_l), g(i_1) \ldots g(i_k), i, j} \tX_{i_1\ldots i_k, i}\\
 &=
 \sum_{i_1\ldots i_k, i} \tL_{j_1\ldots j_l, i_1 \ldots i_k, i, j} \tX_{i_1\ldots i_k, i}, 
 \end{align*}
 for all $g\in G$ and $\tX\in\Real^{n^k\times a}$. This implies \eqref{e:equi_fixed_point_linear}, namely
 $$g\cdot \tL = \tL, \ g\in G.	$$
 The constant part of $L$ is done similarly.

 \end{document}